\documentclass{applemlr}          

\makeatletter
\newif\ifneurips
\@ifclassloaded{applemlr}{\neuripsfalse}{\neuripstrue}
\makeatother

\ifneurips
  \usepackage[preprint]{neurips_2026} 
  \usepackage{etoolbox}                
  \usepackage[table,dvipsnames]{xcolor}
  \usepackage[most]{tcolorbox}         
  \usepackage[colorlinks,linkcolor=blue,citecolor=blue,urlcolor=blue]{hyperref}
  \definecolor{fgcolor}{HTML}{333333}
  \definecolor{bgcolor}{HTML}{EEEEEE}
\fi

\usepackage{amsmath}
\usepackage{enumerate}
\usepackage{algorithm}
\usepackage{algpseudocode}
\usepackage{amsfonts}
\usepackage{amsthm}
\usepackage{cleveref}
\usepackage{diagbox}
\usepackage{colortbl}
\usepackage{amssymb}
\usepackage{xspace}
\usepackage{wrapfig}
\usepackage{adjustbox}
\usepackage{tabularx}
\usepackage{booktabs}
\usepackage{mathtools}
\usepackage{tikz}
\usepackage{enumitem}
\usepackage{silence}
\usepackage{dsfont}
\usepackage[table]{xcolor}
\usepackage[dvipsnames]{xcolor}
\usepackage{multirow}
\usepackage{makecell}
\usepackage{xfakebold}

\usepackage{amsmath,amsfonts,bm}

\def\eqref#1{equation~\ref{#1}}

\def\1{\bm{1}}

\DeclareMathAlphabet{\mathsfit}{\encodingdefault}{\sfdefault}{m}{sl}
\SetMathAlphabet{\mathsfit}{bold}{\encodingdefault}{\sfdefault}{bx}{n}

\definecolor{textgray}{HTML}{6E6E73}
\usetikzlibrary{positioning, calc}
\usetikzlibrary{decorations.pathmorphing}

\makeatletter
\patchcmd{\wrong@fontshape}{\@gobbletwo}{}{}{}
\makeatother
\numberwithin{equation}{section}
\makeatletter
\AtBeginDocument{
  \urlstyle{sf}
  
}
\makeatother

\definecolor{light}{RGB}{125, 125, 125}
\crefname{tcb@cnt@pbox}{code}{code}
\Crefname{tcb@cnt@pbox}{Code}{Code}
\crefname{assumption}{assumption}{assumption}
\Crefname{assumption}{Assumption}{Assumptions}

\newtcolorbox[auto counter]{pbox}[2][]{
  colback=white,
  title=Code~\thetcbcounter: #2,
  #1,fonttitle=\sffamily,
  fontupper=\sffamily,
  arc=2pt,
  colframe=bgcolor,
  coltitle=fgcolor,
  colbacktitle=bgcolor,
  toptitle=0.25cm,
  bottomtitle=0.125cm
}

\makeatletter
\newcommand\applefootnote[1]{%
  \begingroup
  \renewcommand\thefootnote{}%
  \renewcommand\@makefntext[1]{\noindent##1}%
  \footnote{#1}%
  \addtocounter{footnote}{-1}%
  \endgroup
}
\makeatother

\definecolor{cverbbg}{gray}{0.90}

\usepackage[utf8]{inputenc}
\usepackage[T1]{fontenc}
\usepackage{graphicx}
\usepackage{url}
\usepackage{microtype}
\usepackage{placeins}
\usepackage[most]{tcolorbox}

\definecolor{cbpurple}{HTML}{785EF0}
\definecolor{cbpurplebg}{HTML}{F0EDFD}

\newtcolorbox{resultbox}[1][]{
    enhanced,
    breakable,
    colback=cbpurplebg,
    colframe=cbpurple,
    boxrule=1pt,
    arc=3pt,
    left=10pt, right=10pt, top=8pt, bottom=8pt,
    fonttitle=\bfseries\sffamily,
    coltitle=white,
    colbacktitle=cbpurple,
    title=#1,
    attach boxed title to top left={xshift=10pt, yshift=-\tcboxedtitleheight/2},
    boxed title style={arc=2pt, boxrule=0pt},
    drop fuzzy shadow=cbpurple!30!white,
  }

\newtheorem{theorem}{Theorem}
\newtheorem{proposition}[theorem]{Proposition}
\newtheorem{definition}[theorem]{Definition}
\newtheorem{corollary}[theorem]{Corollary}

\title{Limits of Confidence in Diffusion}

\newcommand{\abstracttext}{%
Discrete diffusion, including remasking and uniform-state samplers, generate a sequence by writing multiple token positions per step, drawing each from a per-position distribution and choosing which positions to write from those same distributions.  For domains of general interest (pixels, phonemes, or words) there are inherent dependencies between tokens. We show that a step matches the training distribution only when the positions it writes are conditionally independent given the tokens already fixed, that no product of per-position distributions can match a dependent group, and that per-position distributions do not determine whether a group is dependent:  two joint distributions can have identical per-position marginals while differing in which combinations of values occur. On ScanAndAdd, a synthetic task whose joint distribution is available in closed form, we verify that every group of two or more undetermined positions a confidence ranking writes is dependent, and measure the generated distribution to be $29\times$ the sampling-noise floor total variation while per-sample metrics are $1.0$.%
}

\ifneurips
  \author{%
    Russ Webb \quad Amitis Shidani \quad Alice Bizeul \quad Dan Busbridge \\
    Apple \\
    \texttt{\{rwebb, dbusbridge, amitis\_shidani, abizeul\}@apple.com}
  }
  \date{}
\else
  \author{Russ Webb}
  \author{Amitis Shidani}
  \author{Alice Bizeul}
  \author{Dan Busbridge}
  \affiliation{Apple}
  \abstract{\abstracttext}
  \metadata[Correspondence]{\sffamily \{rwebb, dbusbridge, amitis\_shidani, abizeul\}@apple.com}
  \date{\sffamily\today}
\fi

\begin{document}

\maketitle

\ifneurips
\begin{abstract}
\abstracttext
\end{abstract}
\fi

\section{Introduction}
\label{sec:intro}

A discrete diffusion model builds a sequence in a fixed number of steps, writing several token positions
at each one and choosing which positions to write from the model's own per-position predictions
\citep{austin2021d3pm,chang2022maskgit,sahoo2024mdlm,wang2025remdm}. Writing several positions per step
is the reason these models are of interest, in order to produce a sample in fewer steps. We show that the decision of what to write simultaneously cannot be made from the per-position distributions common approaches read: a step that writes two dependent positions produces the wrong distribution by at least the total correlation (TC) of the group it writes, and this distribution shift exists in fully converged models. 

This distribution shift is not visible via single sample metrics like correctness, so we work on a task whose joint distribution is available in closed form, where a distortion in density can be distinguished from sample diversity. A
model trained on ScanAndAdd reaches $1.00$ well-formedness, correctness, and uniqueness early in training. Figure~\ref{fig:frontier} shows the total variation of tokens (TVT) from one trained model. When the positions written at each step are chosen from the model's own confidences, the generated distribution stays far from the training distribution; although the model reaches $1.0$ correctness it sits $29$ times the sampling-noise floor away in TVT. Driven instead by a hand-specified write order, the same model produces a token distribution that TVT cannot separate from the training data, at correctness $0.987$ to $0.996$. Section~\ref{sec:measured} gives further explanation and details.

\begin{figure}[t]
\centering
\includegraphics[width=0.7\textwidth]{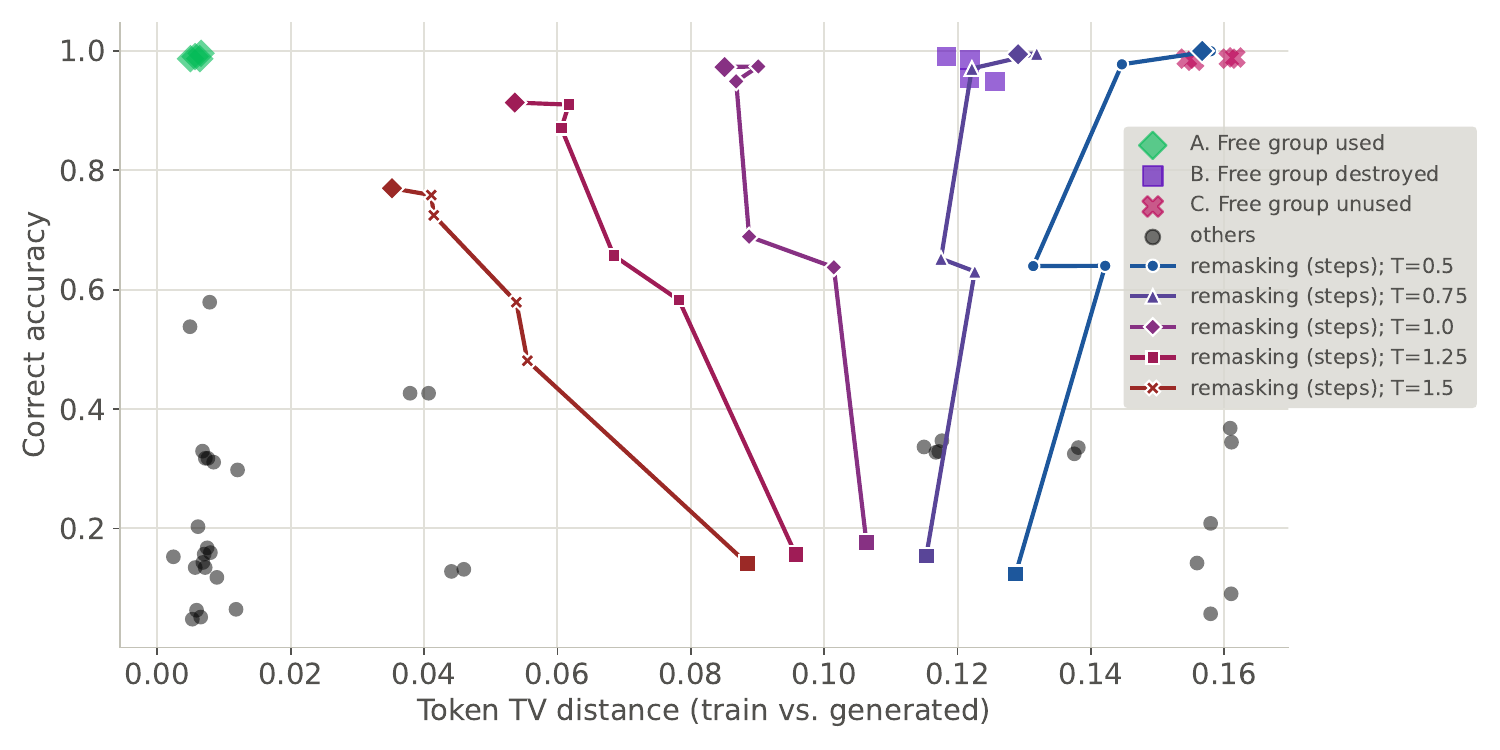}
\caption{Correctness and token TV distance for fixed write orders (scatter) and
confidence-ordered remasking (lines), from the same converged model with $2k$ samples per point. The scatter is grouped by what each ordering does with the sample values (the free group) that may be written together when no answer digit has been
written (Section~\ref{sec:instance}). Group~A \emph{uses} the free group: values in parallel, commands one position
at a time, answer last. Group~B \emph{destroys} the free group by making it entangled: the answer is written before the values, which couples the two the head reads. Group~C samples the free group sequentially: values and commands one position at a time, answer last. Only group~A reaches the noise floor, $\mathrm{TV}\approx0.0054$ at the left edge;
groups~B and~C stay correct at roughly $22$ and $29$ times the sampling floor. Each line sweeps step count $T\in\{5,10,15,20,25,30\}$ at one temperature,
$\tau\in\{0.5,0.75,1.0,1.25,1.5\}$; squares are at $T=5$ and diamonds at $T=30$.}
\label{fig:frontier}
\end{figure}


\paragraph{Roadmap.} Section~\ref{sec:claim} states the claim as one definition of a step covering three
common diffusion families and four results about it; the rest of the paper validates these claims with ScanAndAdd,
a synthetic task whose conditionals, entropies, and support are all computable.

\section{When Matching the Training Distribution is Impossible}
\label{sec:claim}

\textbf{A step that writes two positions which are dependent given the tokens already fixed produces
the wrong distribution.} 
The divergence is at least the total correlation of the group the step writes and persists even when every per-position prediction by the model is exact. Nothing about the data is assumed beyond the dependence between positions (Theorem~\ref{thm:step}). \textbf{No collection of per-position distributions determines which groups are independent}, so a sampler reading them cannot certify any group of two or more undetermined positions (Theorem~\ref{thm:cert}). The states a dependent step produces are ones at which the training objective leaves the conditionals undetermined (Theorem~\ref{thm:dich}), and under hard constraints a sampler that writes one position at a time cannot change a position the rest of the sequence pins, so the valid sequences that differ there are unreachable (Proposition~\ref{prop:frozen}).

Let $p$ be a distribution on $\mathcal V^L$, and a \emph{position} is one of the $L$ token positions of the
sequence.

\begin{definition}[Independent-update step and sampler]\label{def:step}
A sampler has a state, which assigns to each position either a token of $\mathcal V$ or the mask
symbol, and repeatedly applies the following step. First the network supplies one distribution $\pi_i$
over $\mathcal V$ at every position the step is allowed to write, computed from the information
$x_{\mathrm{given}}$ that the algorithm gives it, which is specified per family below. The \emph{scheduler}
then selects the set $R$ of positions to write, which we call a \emph{group}, and the step draws the new value at each $i\in R$ independently from $\pi_i$, leaving every other position unchanged, written $x_{-R}$ and called the \emph{frozen} tokens. The
step's law on $x_R$ is therefore the product $\prod_{i\in R}\pi_i(x_i)$, and the step is
\emph{parallel} if $|R|\ge2$. 

The only assumption made about the scheduler is that its choice of $R$ depend on
$p$ only through per-position distributions of this kind, whether from the current pass or an earlier
one: it may rank positions by them, or by the probability of a token drawn from them, or use any
randomization independent of $p$. We call a sampler whose every step has this form an
\emph{independent-update sampler}.
\end{definition}

The three common diffusion methods differ only in $x_{\mathrm{given}}$:
\begin{itemize}
\item \emph{Absorbing-state masked diffusion} \citep{austin2021d3pm,chang2022maskgit,sahoo2024mdlm}.
$R$ is a set of masked positions, written once and never revisited. Here $x_{\mathrm{given}}$ is
the complement of the currently masked set. Mask symbols carry no information.
\item \emph{Remasking} \citep{ghazvininejad2019maskpredict,wang2025remdm}. The same, except that $R$
may contain positions that already hold tokens. In practice, the step returns those positions to the mask symbol rather than drawing replacement values in place, and a
later step fills them; the results below cover whichever step draws the values, since returning a position to the mask symbol commits to no distribution over
$\mathcal V$. Both the distributions the sampler uses to determine $R$ and those (possibly different) used to determine the values written there are per-position
distributions computed from the state, which is all the results below require.
\item \emph{Uniform-state diffusion} \citep{austin2021d3pm,lou2024sedd}. There is no mask token, so
every position always holds a token and $x_{\mathrm{given}}$ is the entire current sequence. For this family, we read the claims below at the terminal phase of the trajectory.
\end{itemize}


We assume each $\pi_i$ to be as good as training can make it, the exact conditional from $x_{\mathrm{given}}$, which for the first two families is the population minimizer of the per-position masked cross-entropy loss these models are trained with \citep{devlin2019bert,austin2021d3pm,sahoo2024mdlm}. The per-step results below do not depend on how the scheduler ranks positions: Definition~\ref{def:step} admits any rule computed from the $\pi_i$, from tokens drawn from them, or from randomness independent of $p$ --- entropy or margin as much as confidence --- and the proof of Theorem~\ref{thm:cert} uses only the distribution of such a quantity, not how it is computed. What the rule affects is how those per-step costs accumulate over a run (Appendix~\ref{app:onlyworse}). \emph{Confidence} is the quantity ranked in practice, where two forms are common: $\max_v \pi_i(v)$, the largest probability the distribution at position $i$ assigns to any token, and $\pi_i(t_i)$, the probability of the token $t_i$ the sampler has just drawn there, which is what MaskGIT-style and remasking implementations rank. Table~\ref{tab:conf} and Section~\ref{sec:instance} use the first, since at the all-masked state no token has yet been drawn.
  
\begin{definition}[Entangled set]\label{def:ent}
The set $R$ is \emph{entangled} at $x_{-R}$ if there is some combination of tokens that the positions of $R$
can each take separately but not together: there is an assignment $x_R$ with
$p(x_i\mid x_{-R})>0$ for every $i\in R$, while $p(x_R\mid x_{-R})=0$. In the two-position case, both
$a$ at $i$ and $b$ at $j$ occur in the data given $x_{-R}$, but not in the same sequence. Entanglement implies dependence --- independence would give $p(x_R\mid x_{-R})=\prod_{i\in R}p(x_i\mid x_{-R})>0$ --- but not conversely, since a dependent set can have full support. A \emph{free} set is one that is not entangled.
\end{definition}

\begin{theorem}[No product matches a dependent group]\label{thm:step}
If the positions of $R$ are not mutually independent under $p(\cdot\mid x_{-R})$, then no product
$\prod_{i\in R}\pi_i$ equals $p(x_R\mid x_{-R})$, whatever the $\pi_i$ are and whatever
$x_{\mathrm{given}}$ they were computed from. Quantitatively,
\[
\begin{aligned}
D_{\mathrm{KL}}\!\Big(p(x_R\mid x_{-R})\;\Big\|\;\prod_{i\in R}\pi_i\Big)
&=\underbrace{\mathrm{TC}(X_R\mid x_{-R})}_{\text{grouping}}
+\sum_{i\in R}\underbrace{D_{\mathrm{KL}}\big(p(x_i\mid x_{-R})\,\big\|\,\pi_i\big)}_{\text{per-position error at }i},\\[2pt]
\mathrm{TC}(X_R\mid x_{-R})&=\sum_{i\in R}H(X_i\mid x_{-R})-H(X_R\mid x_{-R}),
\end{aligned}
\]
so the divergence is at least $\mathrm{TC}(X_R\mid x_{-R})>0$, with equality exactly when every
$\pi_i$ is the true marginal $p(x_i\mid x_{-R})$. If in that best case $R$ is entangled at $x_{-R}$,
the step assigns positive probability to a state outside $\operatorname{supp}(p)$.
\end{theorem}

\begin{proof}
Expanding the divergence and using that $\sum_{x_R}p(x_R\mid x_{-R})\log \pi_i(x_i)$ depends on
$p(\cdot \mid x_{-R})$ only through its $i$th marginal gives the displayed decomposition; both terms
are non-negative and the first is zero only under mutual independence \citep{watanabe1960,cover2006}.
For the support claim, take $x_R$ coordinatewise admissible with $p(x_R\mid x_{-R})=0$; the step
assigns it $\prod_i p(x_i\mid x_{-R})>0$, and $(x_{-R},x_R)$ has probability zero under $p$.
\end{proof}

Better training shrinks the per-position terms only, so the grouping term cannot be reduced by training --- a floor \citet{zhang2026generationorderparalleldecoding} establish independently for exact marginals, and which the decomposition above separates from per-position error. 
The floor exists for uniform-state diffusion too, even though its $\pi_i$ are not marginals. There $x_{\mathrm{given}}$ is the entire current sequence, including the tokens the step is about to overwrite at the positions of $R$, so $\pi_i$ is a conditional given those stale tokens rather than the marginal of $p(\cdot\mid x_{-R})$. That gap is a per-position error, and the per-position terms are non-negative, so the divergence between the step's law and $p(\cdot\mid x_{-R})$ can only increase. Exactness therefore requires the sampler to write only groups whose positions are conditionally
independent given the frozen tokens. Whether a group has that property is a fact about the joint law
of the group. What the sampler reads is one distribution per position, and no collection of
per-position distributions records how two positions co-vary.  For a sampler that writes each position once, these per-step costs add and cannot cancel (Appendix~\ref{app:onlyworse}).

\begin{theorem}[No collection of per-position distributions can certify a group]\label{thm:cert}
Let the scheduler read the distributions supplied at the positions the step may write and select a set
$R$ containing at least two positions $i,j$ at which $\pi_i,\pi_j$ are not point masses. Then there are two
distributions on $\mathcal V^L$ whose per-position marginals are equal at all $L$ positions: one under which the positions of $R$ are mutually independent given $x_{-R}$, and one under which $R$ is entangled at $x_{-R}$. Everything
the scheduler reads has the same law under both, including tokens drawn from those distributions, so it selects the same $R$ with the same probability, and the step over that $R$ is exact under one and inexact under
the other.
\end{theorem}

\begin{proof}
Let $p$ make the positions of $R$ mutually independent with the marginals the scheduler observed, and
make $X_R$ independent of the positions outside $R$; the reveal of $R$ is then exact. Choose $a\neq b$
in $\operatorname{supp}(\pi_i)$ and $c\neq d$ in $\operatorname{supp}(\pi_j)$, write $p_{ij}$ for the
joint law of $(X_i,X_j)$ under $p$, and let $\varepsilon=\min\{p_{ij}(a,d),p_{ij}(b,c)\}>0$. Let $p'$
replace $p_{ij}$ by
\[
\begin{aligned}
p'_{ij}(a,c)&=p_{ij}(a,c)+\varepsilon, & p'_{ij}(b,d)&=p_{ij}(b,d)+\varepsilon,\\
p'_{ij}(a,d)&=p_{ij}(a,d)-\varepsilon, & p'_{ij}(b,c)&=p_{ij}(b,c)-\varepsilon,
\end{aligned}
\]
unchanged on the other value pairs, and leave the rest of $p$ as it is. Row and column sums are
preserved, so every per-position marginal is unchanged. One of $p'_{ij}(a,d)$, $p'_{ij}(b,c)$ is zero
while its two values remain individually admissible, so $R$ is entangled under $p'$. Since the $\pi_i$
coincide under $p$ and $p'$, so does the law of any quantity the scheduler computes from them or from
tokens drawn from them, including the confidence of a drawn token; and $X_R$ is independent of the other
positions under both, so those distributions coincide in every context at which the scheduler could still select an $R$ containing both. The choice of $R$ therefore has the same law under both, and by Theorem~\ref{thm:step} the step over $R$ is
exact under $p$ and not under $p'$.
\end{proof}


A sampler that can revisit a position it has already written --- a remasking or uniform-state sampler --- may try
to repair a step that left the support. In either family, the decision is made at a state to which $p$ assigns
probability zero: which positions to remask in the first case, what to write in the second. 

\begin{theorem}[After an error, the conditionals are either revealing or undefined]\label{thm:dich}
Let $x$ satisfy $p(x)=0$. Since $p(x)=p(x_{-i})\,p(x_i\mid x_{-i})$, at each position $i$ at least one
factor vanishes, so either
\begin{itemize}
\item[(a)] $p(x_{-i})>0$ and $p(x_i\mid x_{-i})=0$: position $i$ reports probability zero for its current token, so any confidence-based rewrite rule selects $i$; or \item[(b)] $p(x_{-i})=0$: the tokens outside $i$ are themselves impossible, and $p$ defines no conditional distribution $p(\cdot\mid x_{-i})$.
\end{itemize}
If (b) holds at every position, nothing the sampler reads at this state is determined by $p$.
\end{theorem}

Branch (b) is the direct result of writing an entangled group: the surrounding tokens at every position of the group are themselves impossible. The masked cross-entropy objective trains only on contexts that arise from corrupting valid samples; probability-zero contexts do not arise that way. Training with token-substitution corruption reaches some such states, but fixes the corruption posterior, not a conditional of $p$. A single wrong position falls under (a): the surrounding tokens are valid and the network reports zero for $i$'s current token and a low-probability rewrite rule selects it. The correlated multi-position errors an entangled group produces all fall under (b) at every position. Repair in branch (a) is also imperfect. Rewriting one member of an entangled pair reduces parallelism, and the redraw is exact only if the surviving member was already drawn from the correct conditional, which fails after any prior parallel step.

When a remasking or uniform-state sampler reaches a valid sequence with steps remaining, changing the distribution over valid sequences requires exchanging it for another. Where the data pins a position, writing that position alone leaves the support, and any group that could change it is entangled (Proposition~\ref{prop:frozen}): no such exchange is exact.

\begin{proposition}[Writing a pinned position either leaves the support or writes an entangled group]\label{prop:frozen}
Call position $i$ \emph{pinned} at $x$ if $p(\cdot\mid x_{-i})$ is a point mass at $x_i$, that is, if
the current tokens elsewhere admit no other value there. Let $x\in\operatorname{supp}(p)$ and let $i$ be
pinned at $x$.
\begin{itemize}
\item[(a)] Every state that differs from $x$ only at $i$ lies outside $\operatorname{supp}(p)$, since
$p(y)=p(x_{-i})\,p(y_i\mid x_{-i})=0$.
\item[(b)] Let a step write a set $R\ni i$. With the exact conditionals of
Definition~\ref{def:step}, the step can put a value other than $x_i$ at $i$ only if
$p(\cdot\mid x_{-R})$ is non-degenerate there, and in that case $R$ is entangled at $x_{-R}$: changing
$i$ and leaving the other positions of $R$ as they are is admissible at every coordinate and impossible
jointly, by (a). Theorem~\ref{thm:step} bounds the step's divergence, and it puts mass outside
$\operatorname{supp}(p)$.
\end{itemize}
An independent-update sampler therefore reaches a valid sequence differing from $x$ at a pinned
position only by leaving the support, where Theorem~\ref{thm:dich} governs what it reads next, or by
writing an entangled group, where Theorem~\ref{thm:step} applies. Drawing $x_R$ from the joint
conditional $p(x_R\mid x_{-R})$ would move the state exactly, but that is not a product and so not one
of its steps. This argument is about $p$ alone, so it holds whatever
$x_{\mathrm{given}}$ the family supplies to the network, uniform-state diffusion included: conditioning
on more of the current sequence, its own token at $i$ among it, cannot make a determined position
undetermined.
\end{proposition}

\begin{corollary}[Exactness requires independence the sampler cannot verify]\label{cor:main}
An independent-update sampler reproduces $p$ only if every group of two or more positions it writes is
conditionally independent given the frozen tokens; Appendix~\ref{app:onlyworse} shows when that follows
from Theorem~\ref{thm:step} step by step. Conditional independence is a property of the particular $p$
the model was trained on, and the sampler cannot verify it: by Theorem~\ref{thm:cert} any group with two
positions at which the distributions are non-degenerate is entangled under some distribution with
exactly the per-position marginals the scheduler observed, and by Theorem~\ref{thm:step} writing an
entangled group puts mass outside the support.
\end{corollary}

\begin{corollary}[Overwriting does not recover the distribution]\label{cor:repair}
At the off-support states a parallel step produces, $p$ either reports probability zero for a single
wrong token, which any low-probability rewrite rule locates, or determines nothing at all, which is what
the correlated multi-position errors of an entangled group produce (Theorem~\ref{thm:dich}). And where
$p$ carries hard constraints, changing a position the rest of the sequence pins takes the sampler
outside the support unless it writes an entangled group, which Theorem~\ref{thm:step} then applies
(Proposition~\ref{prop:frozen}). On such a $p$, writing one position at a time cannot reach a different valid sequence, and writing multiple positions distorts the distribution.
\end{corollary}
  
\section{Experimental Verification}
\label{sec:task}

Section~\ref{sec:claim} holds over any $p$. Checking it against one needs a distribution whose
conditionals, entropies, and support are all computable. That distribution comes from a synthetic task, ScanAndAdd: a read head sweeps once, left to right, over $n$ values $v_0,\dots,v_{n-1}\sim\mathrm{Uniform}\{0,\dots,V\}$, following $A+(n-1)$ commands: an \texttt{add} adds the value under the head to an accumulator, a \texttt{right} moves the
head one position right. Each sample holds exactly $A$ \texttt{add} and exactly $n-1$ \texttt{right}
commands in uniformly random order, and the answer is the final accumulator. A sample is serialized
into three marked fields,
\[
\texttt{VALS}\;\; v_0 \cdots v_{n-1} \;\;\; \texttt{ANS}\;\; a_2\,a_1\,a_0 \;\;\; \texttt{OPS}\;\; o_0 \cdots o_{A+n-2},
\]
with the answer in three base-$10$ digits. All experiments use $n=9$, $V=9$, $A=2$, giving $L=25$ over
a $15$-token vocabulary; commands are read by parity from token ids $\{0,\dots,M\}$, so $M=1$ gives one
id per command and $M=9$ gives five interchangeable ids for each
(Appendix~\ref{app:task}). \texttt{VALS} and \texttt{OPS} are independent roots and \texttt{ANS} is a
deterministic function of both. The distribution $p$ is uniform on its support of
$|S|=(V+1)^n\binom{A+n-1}{A}=4.5\times10^{10}$ sequences at $M=1$, so $H(p)=35.39$ bits.

Because $p$ is known, the per-position distributions a perfectly trained model reports at the
all-masked state are available in closed form. Table~\ref{tab:conf} lists them, derived in
Appendix~\ref{app:conf}. Figure~\ref{fig:setup} shows the generative process and its non-distributional metrics, all of which the task passes.

\begin{figure}[t]
\centering
\begin{minipage}[c]{0.50\textwidth}
\centering
\includegraphics[width=\linewidth]{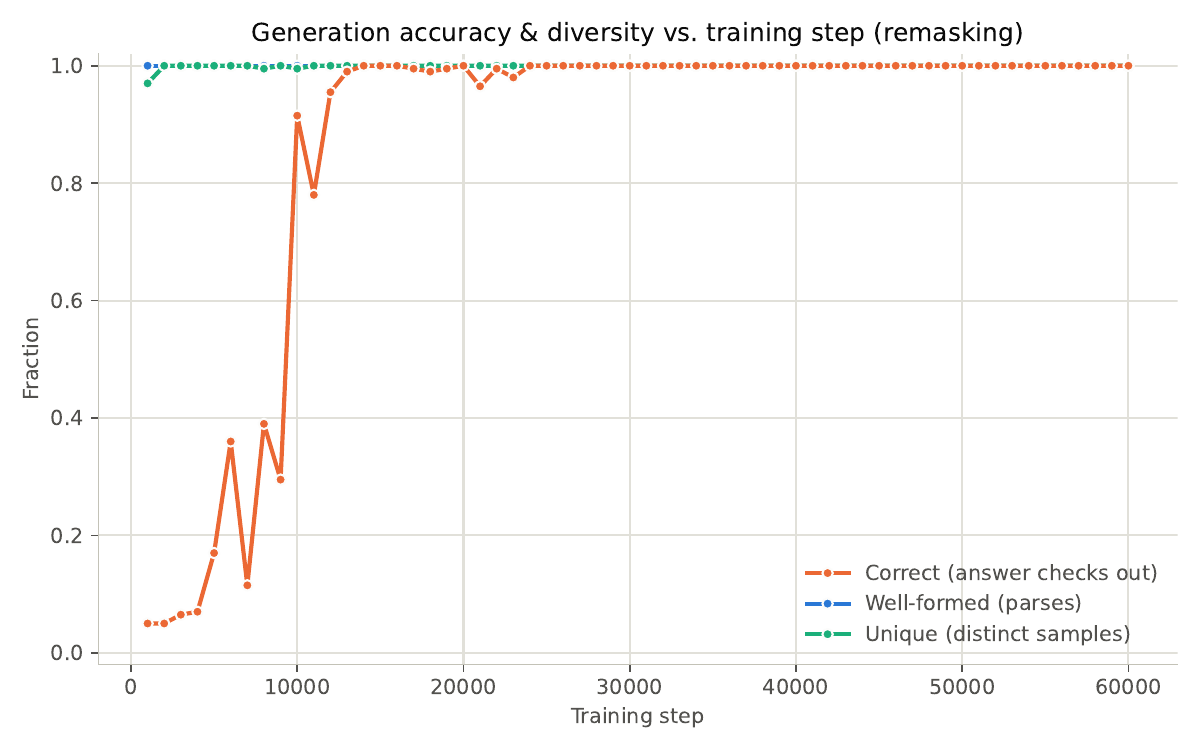}
\end{minipage}\hfill
\begin{minipage}[c]{0.46\textwidth}
\footnotesize
\begin{verbatim}
add_ops = arange(0, M+1, 2)
right_ops  = arange(1, M+1, 2)
vals = choice(arange(0, V+1), n)
ops = concatenate([
    choice(add_ops, A), 
    choice(right_ops, n-1)]) 
shuffle(ops)

total, head = 0, 0
for op in ops:
    if op in right_ops:
        head = (head + 1) % n
    else:                     
        total += vals[head] % m
\end{verbatim}
\end{minipage}
\caption{\textbf{Left:} well-formedness, correctness, and uniqueness for the training ($T=24$, $\tau=0.40$, $200$ samples). All three saturate at $1.0$ early in training. \textbf{Right:} the ScanAndAdd process, for
$n$ values in $[0,V]$, $A$ \texttt{add} commands, command token ids $\{0,\dots,M\}$, and value
modulus $m$. The values are drawn i.i.d.; the command list is a shuffle of $A$ \texttt{add}
and $n-1$ \texttt{right} ids, so positions are coupled;
\texttt{total} is a deterministic function of the other two fields.}
\label{fig:setup}
\end{figure}

\begin{table}[t]
\centering
\small
\caption{Maximum per-position probability $\max_v p(v\mid\varnothing)$ at the all-masked state, in
closed form ($n=9$, $V=9$, $A=2$). A rank-based scheduler writes positions top to bottom. Among
content positions only $a_2$ is determined, yet $a_1$ and, at $M=1$, the commands outrank every value
position.}
\label{tab:conf}
\begin{tabular}{llrr}
\toprule
\textbf{Position} & \textbf{Support} & $\max_v p(v\mid\varnothing)$ & $H(\cdot\mid\varnothing)$ (bits) \\
\midrule
$a_2$ (ANS hundreds) & $\{0\}$ & $1.00$ & $0.00$ \\
OPS position, $M=1$ & $\{0,1\}$ & $0.80$ & $0.72$ \\
$a_1$ (ANS tens) & $\{0,1\}$ & $0.54$ & $0.995$ \\
OPS position, $M=9$ & $\{0,\dots,9\}$ & $0.16$ & $3.04$ \\
$a_0$ (ANS units) & $\{0,\dots,9\}$ & $0.12$ & $3.29$ \\
$v_i$ (VALS position) & $\{0,\dots,9\}$ & $\mathbf{0.10}$ & $3.32$ \\
\bottomrule
\end{tabular}
\end{table}

\section{Applicability to ScanAndAdd}
\label{sec:instance}

Corollary~\ref{cor:main} requires the positions of every group a sampler writes to be mutually independent given the frozen tokens. This section shows that on ScanAndAdd no confidence ranking meets that requirement. Confidence at the all-masked state orders the positions as in Table~\ref{tab:conf}, lowest are the value positions and highest the field markers and $a_2$.

\emph{The command positions are dependent}, since exactly two of the ten are \texttt{add}:
$\mathrm{TC}=0.0099$ bits for a pair and $1.727$ bits for the set. Three or more are also entangled,
because each of three positions can hold \texttt{add} on its own while all three cannot; drawing all ten
independently gives a block with an \texttt{add} count other than two with probability $0.698$.  \emph{Both non-degenerate answer digits are dependent} at $\mathrm{TC}=0.189$ bits and entangled, since
$a_1=1$ and $a_0=9$ each occur but never together: that would record $19$, which is above the largest
attainable answer of $18$. \emph{An answer digit together with a value or command position it affects are entangled}, because the recorded answer is a deterministic function of values and commands, so most combinations that are separately possible are jointly impossible. \emph{All value positions are mutually independent as long as no answer digit has been written}. The value positions (the free group) are the only undetermined ones in the task that a step may write together.

\paragraph{How each diffusion family fails.} The failure of each family follows from one of the dependencies above; details are in Appendix~\ref{app:algorithms}.
\begin{itemize}
\item \emph{Masked reveal} (Appendix~\ref{app:masked}). The free group carries the lowest confidence in
the sequence, so a ranking reaches it last --- after the recorded answer has made two of its positions
dependent. On $72\%$ of samples those two then
outrank the seven that remain free, so they are what the next parallel step writes. Until then, every
group of two or more undetermined positions is drawn from the command block and the answer digits, and
those have non-zero cost.
\item \emph{Remasking} (Appendix~\ref{app:remask}). Overwriting requires a per-position conditional,
which exists only where the surrounding tokens are themselves producible. A parallel write of the
command block fails that condition at every position at once with probability $0.228$, which puts the
sampler in branch (b) of Theorem~\ref{thm:dich}: the ranking that selects what to overwrite is then
not determined by $p$.
\item \emph{Uniform-state diffusion} (Appendix~\ref{app:uniform}). For valid sequences, every
position except the value positions the head does not read is pinned by the others, so by
Proposition~\ref{prop:frozen} changing the command arrangement either leaves the support or writes an
entangled group. Its terminal phase has no exact route from one arrangement to another.
\end{itemize}

Every step of a confidence-ordered sampler on ScanAndAdd that writes two or more undetermined positions
is therefore inexact: the three markers and $a_2$ are the only positions it may write together at no
cost. The steps that write three or more command positions, both answer digits, or the summed pair leave
the support. 

What Corollary~\ref{cor:main} leaves open is a $p$ on which the groups a sampler happens to write are
conditionally independent. ScanAndAdd has one, the free group, but a confidence ranking reaches it last,
so the exception is unavailable to any confidence ranking in the three families. A masked sampler that writes one position
per step in an order fixed in advance is outside the claim: it is a chain-rule factorization of $p$, and
exact, at $L$ steps, but does not achieve diffusion's key promise of generation with fewer steps.
Appendix~\ref{app:trace} confirms that a trained model orders its writes as Table~\ref{tab:conf} predicts, and Appendix~\ref{app:schedule} works out how much parallelism the one exception permits here.

\section{Measuring Distortion on a Trained Model}
\label{sec:measured}

The model is a $6$-layer bidirectional Transformer encoder \citep{vaswani2017attention},
$d_{\text{model}}=256$, trained $60k$ steps on the $9\times10^6$ training sequences of a $10^7$-sample
dataset with a masked-language-modeling objective that draws a fresh corruption ratio per example. The training with
batch size $128$ gives $0.85$ epochs. Every point in Figure~\ref{fig:frontier} uses the same fully-trained model and $2k$ generated samples.

We report the total variation distance between the pooled tokens (TVT) of generated and training
samples, versus the TVT of a perfect sampler, $0.0054$ computed from independent draws of training data. The scatter of fixed write orders uses $\tau=1.0$; the overlaid lines sweep step count at five temperatures. With $2k$ samples the standard error of a correctness estimate is at most $0.011$.

Group~A in Figure~\ref{fig:frontier} is hand-specified orders sampling values (the free group) in parallel, commands one at a time, and answers last, to reach $\mathrm{TVT}$  $0.0050$ to $0.0066$ at
correctness $0.987$ to $0.996$ with bootstrap $p$-values from $0.18$ to $0.57$: at this sample size a test cannot separate their token distributions from a fresh draw of training data. These orders are not available to a
confidence-ordered sampler, which by Table~\ref{tab:conf} writes the values last; they measure an
achievable frontier, not a decoding rule.

\paragraph{Confidence-ordered remasking with the same model is imperfect.} The best TV at
each temperature is $0.129$, $0.115$, $0.0851$, $0.0536$ and $0.0352$ for $\tau=0.5$ through
$\tau=1.5$, reached at correctness $0.124$, $0.154$, $0.973$, $0.914$ and $0.770$ respectively. The
setting that comes closest to the training distribution is not solving the task. At
correctness $0.987$ or above --- the range the hand-specified orders occupy --- no
setting does better than $\mathrm{TV}=0.129$, which is $24\times$ the floor. The two settings that
reach correctness $1.000$ and $0.9995$ both sit at $\mathrm{TV}\approx0.157$. At
that first setting ($\tau=0.5$, $T=30$) the model is $100\%$ well-formed, $100\%$ correct and $100\%$
unique on $2k$ samples, and all those samples are inside the support of the training
distribution, while their pooled token distribution is $29\times$ TV noise floor. The
distance between the best-TV fixed order and the best-correctness confidence setting is $1.3$ points of
correctness and $32\times$ in TV: correctness does not order these samplers.

\paragraph{Moving one field, with everything else identical, changes correctness by up to
$8.4\times$.} To isolate the effect of where the recorded answer is written, take pairs of orders that
treat the value and command positions identically, with the same grouping and the same parallelism, and
differ only in where the three answer digits fall (Table~\ref{tab:matched}). Writing the answer before
the values makes the summed value pair dependent: knowing the answer removes $3.61$ of the $5.49$ bits of entropy in the command block. Under an answer-last order every partial state can be completed, so an error yields a valid sequence from the wrong part of the distribution; under an answer-first order a partial state can admit no valid completion. Which branch of Theorem~\ref{thm:dich} such a state falls in depends on how much of the command block is wrong; Appendix~\ref{app:remask} works out the case that
falls in branch (b).

\begin{table}[ht]
\centering
\small
\caption{Matched pairs of write orders. The value and command positions are treated the same within rows; only the answer digits move. Brackets mark
positions written in one step; unbracketed fields are written one position per step chosen
by confidence.}
\label{tab:matched}
\begin{tabular}{lr|lr|r}
\toprule
\textbf{ANS last} & \textbf{acc.} & \textbf{ANS in the middle} & \textbf{acc.} & \textbf{ratio} \\
\midrule
\texttt{[VOA][v$\times$9]o$\times$10[ht1]} & $0.987$ & \texttt{[VOA][v$\times$9][ht1]o$\times$10} & $0.118$ & $8.4\times$ \\
\texttt{[VOA][v$\times$9]o$\times$10ht1}   & $0.996$ & \texttt{[VOA][v$\times$9]ht1o$\times$10}   & $0.157$ & $6.3\times$ \\
\texttt{[VOA]v$\times$9o$\times$10ht1}     & $0.989$ & \texttt{[VOA]v$\times$9ht1o$\times$10}     & $0.209$ & $4.7\times$ \\
\texttt{[VOA]v$\times$9o$\times$10[ht1]}   & $0.987$ & \texttt{[VOA]v$\times$9[ht1]o$\times$10}   & $0.142$ & $7.0\times$ \\
\texttt{[VOA][v$\times$9][o$\times$10][ht1]} & $0.330$ & \texttt{[VOA][v$\times$9][ht1][o$\times$10]} & $0.052$ & $6.4\times$ \\
\texttt{[VOA]v$\times$9[o$\times$10][ht1]} & $0.345$ & \texttt{[VOA]v$\times$9[ht1][o$\times$10]} & $0.057$ & $6.0\times$ \\
\bottomrule
\end{tabular}
\end{table}

\paragraph{Mass outside the support is measured directly.} Manual sampling order is notated: markers by \texttt{VOA}, values by \texttt{v}, commands by \texttt{o}, and answer digits by \texttt{ht1} (for hundreds, tens, and ones); brackets show parallel sampling. For orders \texttt{[VOA][v$\times$9][o$\times$10][ht1]} and \texttt{[VOA][v$\times$9][ht1][o$\times$10]}, which write the ten command positions in one step, the fraction of generated samples whose command block has an \texttt{add} count other than two is $0.703$ and $0.707$, against the $0.698$ that
Section~\ref{sec:instance} predicts. Those samples have probability zero under $p$, while being well-formed and correct.

\section{Limitations}
\label{sec:limitations}

\emph{Theorem~\ref{thm:step} bounds one step's divergence, and the argument from there to the output distribution is complete for the absorbing family and open for the others.} Appendix~\ref{app:onlyworse} closes it whenever the set written at each step is a function of the tokens already written --- a schedule fixed in advance, or any ranking computed from the per-position distributions themselves: the per-step divergences then add, none can be negative, and the run is at least as wrong as its worst step. What it does not cover is a dependent group formed at a later step of a ranking, where the group and the context are both random --- conditional on reaching a context, which positions are selected and which values are drawn into them are dependent random variables, and we do not rule out cancellation across reveal paths. For remasking and uniform-state samplers, where a position can be written more than once, we show that the decision to overwrite is unfounded (Theorem~\ref{thm:dich}) and that changing a pinned position costs either the support or a dependent group (Proposition~\ref{prop:frozen}), which Corollary~\ref{cor:repair} combines; we do not derive a quantitative gap between their output law and $p$.

\emph{Theorem~\ref{thm:cert} rules out certification, not correctness on a particular distribution.}
It says that no sampler reading per-position distributions can establish that a group is safe. A given
$p$ may still be one on which the groups a sampler happens to write are independent.
Section~\ref{sec:instance} closes that gap for ScanAndAdd by exhibiting the dependencies directly, and
we do not close it for any natural corpus.

\emph{The uniform-state claims are read at the terminal phase and are not measured.} We analyze the
phase in which the frozen tokens are clean, because that is where $p(x_i\mid x_{-i})$ is the quantity a
perfect denoiser supplies. We do not analyze the high-noise phase, so the statement that the emitted
\texttt{OPS} distribution is whatever that phase left is a conjecture and not a result. No
uniform-state model is trained here: every measurement in Section~\ref{sec:measured} is an
absorbing-state or remasking sampler, and all of them are the $M=1$ condition.

\emph{The size of the distortion is measured on one synthetic task.} ScanAndAdd was built so that
every quantity is computable. What does transfer is the hypothesis of
Theorem~\ref{thm:step}, a group holding a dependent pair, and the one measurement that needs no access
to the joint: counting generated samples that violate a hard constraint known to hold in the data.
While findings in the community make compatible observations at large scale \citep{ni2026flexibility, zhang2026generationorderparalleldecoding} and propose methods for independence in parallel sampling \citep{azangulov2025parallelsamplingmaskeddiffusion, ringel2026dependencyguidedparalleldecodingdiscrete}, this work uses synthetic data and research models that do not duplicate the architectures, decoding strategies, or deployment of any product model.

\emph{A scheduler trained on the domain is not ruled out, and is a plausible response to these
results.} Theorem~\ref{thm:cert} assumes the choice of $R$ depends on $p$ only through the $\pi_i$. A
learned scheduler breaks that assumption, because its parameters carry information about $p$ that those
distributions do not. It still faces Theorem~\ref{thm:step}: its steps are products, so it must avoid
dependent groups, and it cannot exceed the parallelism the data itself permits. On ScanAndAdd, any exact
schedule must write the undetermined command positions one per step, since any two of them are
dependent, and by exchangeability the number of command positions still undetermined when written has
the same law under any order that does not read the drawn tokens: $64/9\approx7.1$ on average and $9$
at worst. No exact schedule for this task therefore averages fewer than about $8$ steps against
$L=25$. The expected speedup is capped near
$3\times$, against the $1.31\times$ that the marginal-reading rule of
Appendix~\ref{app:schedule} attains. Such a scheduler would also have to be trained for
distributional match, which is computable here and only a proxy on a corpus.

\section{Related Work}
\label{sec:related}

\textbf{The failure has a direct precedent in non-autoregressive translation.} \citet{gu2018nat}
named the multimodality problem: a model emitting tokens independently cannot represent a multimodal
output distribution, and produces token-level blends of distinct valid outputs. Mask-Predict
\citep{ghazvininejad2019maskpredict} introduced the confidence-ordered refinement that masked diffusion
samplers inherit. Theorem~\ref{thm:step} is that observation stated as an impossibility, under exact
per-position predictions rather than as a symptom of imperfect training, and
Theorems~\ref{thm:cert}~and~\ref{thm:dich} extend it to the choice of what to write and what to
overwrite.

\textbf{The diffusion objective has been refined repeatedly; the grouping rule has not.} 
D3PM \citep{austin2021d3pm} set out the framework, MaskGIT \citep{chang2022maskgit}
introduced confidence-based parallel unmasking, simplified masked diffusion language models
\citep{sahoo2024mdlm} and score-entropy formulations \citep{lou2024sedd} sharpened the objective, and
ReMDM \citep{wang2025remdm} added inference-time remasking. Evaluation in this line relies on
perplexity, sample quality, and downstream metrics. Theorem~\ref{thm:step} applies to all of these
samplers at once, since each draws the positions it writes independently from per-position
distributions, and Definition~\ref{def:step} is what makes that one statement rather than three.
\citet{zhang2026generationorderparalleldecoding} give an information-theoretic account of the same
  grouping cost and its aggregation across blocks; our contribution is Theorem~\ref{thm:cert}, that no
  scheduler reading per-position distributions can determine whether the positions it selects are conditionally independent.

\textbf{An independent measurement of the ordering effect on language models.}
  \citet{ni2026flexibility} compare confidence-ordered decoding against a fixed left-to-right order on
  three diffusion language models and four reasoning benchmarks, and find the confidence order covers a
  smaller subset of the solution space: on HumanEval (LLaDA-Instruct, Pass@$1024$) $21.3\%$ of problems
  are solved only under the fixed order against $0.6\%$ only under the confidence order. Their ordering
  experiments decode one token per step, so the cost of Theorem~\ref{thm:step} is not active there and
  what they measure is premature commitment rather than grouping.

\textbf{The method follows work on checkable algorithmic tasks} used to probe transformer computation
\citep{lee2023teaching,zhou2022algorithmic}.

\section{Conclusion}
\label{sec:conclusion}

Parallel decoding requires choosing which token positions to write together, and that choice is exact
only when the chosen positions are conditionally independent given what is already fixed. Independence
is a property of the joint distribution, while what a sampler reads is one distribution per position, so
a sampler that chooses its groups from those distributions can certify none beyond the positions its
context has already determined. Remasking and uniform-state samplers inherit that limit rather than lifting it: they must either overwrite at states where the training objective stops constraining the network, or leave pinned positions unchanged, making the valid sequences that differ there unreachable. 


Two consequences are notable: 1) sample validity is not evidence of distribution matching. Where the joint
is computable, measure against it with a stated noise floor and a $p$-value; where it is not, the same
distortion should be expected and is going unmeasured. One part of it can be measured without knowing
the joint at all: a step that writes an entangled group leaves the support, so counting the generated
samples that violate a hard constraint known to hold in the data --- a fixed count, a checksum, a
bracket match --- measures the effect directly, as Section~\ref{sec:measured} does here. 2) the
grouping decision needs the data's dependency structure, which no collection of per-position
distributions supplies. Beyond the positions its context has already determined, a sampler's
parallelism has to be justified by the format the data is written in rather than by the model's
confidence at any position.
The grouping decision is not the only departure from exact sampling, and
Theorem~\ref{thm:step} bounds the others in the same terms. Appendix~\ref{app:interventions} places
other diffusion interventions (such as temperature, top-k, min-p, and search) in context and points out that many common techniques also shift the generated distribution away from the training distribution.

\bibliographystyle{plainnat}

\appendix

\section{ScanAndAdd Task, Data, and Training Details}
\label{app:task}

\paragraph{Vocabulary and encoding.}
Token ids $0$--$9$ are literal base-10 digits. Special tokens: \texttt{VALS} $=10$,
\texttt{OPS} $=11$, \texttt{ANS} $=12$, \texttt{PAD} $=13$, \texttt{MASK} $=14$.
Vocabulary size $= 15$. Total sequence length $L = 2n + A + 5 = 25$.

Each command token $o \in \{0, \ldots, M\}$ is read as \texttt{add} if $o$ is even,
\texttt{right} if odd. Answer digits are base-10 with the hundreds, tens, and ones 
places being $a_2, a_1, a_0$ respectively.

\begin{table}[ht]
\centering
\small
\caption{Sequence layout for $n=9$, $A=2$ ($L=25$). The ten command tokens fill
positions $15$--$24$ with no trailing \texttt{PAD}.}
\label{tab:layout}
\begin{tabular}{lll}
\toprule
\textbf{Positions} & \textbf{Field} & \textbf{Contents} \\
\midrule
$0$        & marker        & \texttt{VALS} \\
$1$--$9$   & values        & $v_0, \dots, v_8$ \\
$10$       & marker        & \texttt{ANS} \\
$11$--$13$ & answer digits & $a_2, a_1, a_0$ \\
$14$       & marker        & \texttt{OPS} \\
$15$--$24$ & commands      & $o_0, \dots, o_9$ \\
\bottomrule
\end{tabular}
\end{table}

\paragraph{The sample-level evaluation metrics are described below:}
\begin{itemize}
  \item \textbf{Well-formedness:} the sequence contains exactly one each of
        \texttt{VALS}, \texttt{ANS}, \texttt{OPS} in that order, exactly three
        tokens between \texttt{ANS} and \texttt{OPS}, and every field token in
        its declared range.
  \item \textbf{Correctness:} additionally, the accumulator recomputed from the
        decoded values and commands equals the decoded answer.
  \item \textbf{Neither metric checks the add-command count.} Both can read $1.00$
        on samples whose command block has the wrong number of adds.
\end{itemize}

\paragraph{Worked example ($M = 1$).}
\[
  \mathbf{v} = (4, 9, 6, 3, 3, 7, 7, 9, 7), \qquad
  \mathbf{o} = (0, 1, 1, 1, 0, 1, 1, 1, 1, 1).
\]
Command positions $1$ and $5$ are even (\texttt{add}); the rest are \texttt{right}.
The head starts at position $0$, reads $v_0 = 4$, moves right three times, reads
$v_3 = 3$: answer $= 7$. This sample in serialized form is
\[
  \underbrace{10}_{\texttt{VALS}}\;
  4\; 9\; 6\; 3\; 3\; 7\; 7\; 9\; 7\;\;
  \underbrace{12}_{\texttt{ANS}}\;
  0\; 0\; 7\;\;
  \underbrace{11}_{\texttt{OPS}}\;
  0\; 1\; 1\; 1\; 0\; 1\; 1\; 1\; 1\; 1.
\]
The answer is always below $100$, so the hundreds digit $a_2 = 0$ in every sample.
This is why $a_2$ carries confidence $1.00$.

\begin{table}[h!]
\centering
\small
\caption{Data parameters. The two conditions differ only in $M$.}
\label{tab:params}
\begin{tabular}{llrr}
\toprule
\textbf{Symbol} & \textbf{Config key} & \textbf{$M=1$} & \textbf{$M=9$} \\
\midrule
$n$ (values)          & \texttt{num\_vals}      & $9$    & $9$    \\
$V$ (max value)       & \texttt{max\_val}       & $9$    & $9$    \\
$m$ (value modulus)   & \texttt{val\_mod}       & $10$   & $10$   \\
$A$ (add commands)    & \texttt{num\_add}       & $2$    & $2$    \\
$M$ (max command id)  & \texttt{max\_op}        & $1$    & $9$    \\
Command multiplicity  & ---                     & $1$    & $5$    \\
Sequence length $L$   & \texttt{sample\_length} & $25$   & $25$   \\
Vocabulary size       & ---                     & $15$   & $15$   \\
Training examples     & \texttt{num\_samples}   & $10^7$ & $10^7$ \\
\bottomrule
\end{tabular}
\end{table}

\begin{table}[ht]
\centering
\small
\caption{Training and generation hyperparameters (both conditions, except
\texttt{max\_op}). Parameter count covers the full model; sinusoidal encoding
contributes none. Error samples are the fraction eligible for random-token
corruption.}
\label{tab:appendix}
\begin{tabularx}{\linewidth}{llX}
\toprule
\textbf{Group} & \textbf{Parameter} & \textbf{Value} \\
\midrule
\multirow{2}{*}{Data}
 & Training examples          & $10$M \\
 & Train / validation split   & $0.9 / 0.1$ \\
\midrule
\multirow{7}{*}{Model}
 & Positional encoding        & sinusoidal \\
 & $d_{\text{model}}$         & $256$ \\
 & Attention heads            & $8$ \\
 & Layers                     & $6$ \\
 & Feed-forward dim           & $1024$ \\
 & Dropout                    & $0.0$ \\
 & Trainable parameters       & $4.75$M \\
\midrule
\multirow{7}{*}{Optimizer}
 & Optimizer                  & AdamW \\
 & Peak learning rate         & $1.5 \times 10^{-4}$ \\
 & $(\beta_1, \beta_2)$       & $(0.9, 0.95)$ \\
 & Weight decay               & $0.001$ \\
 & Warmup steps               & $250$ \\
 & LR schedule                & cosine, $0.1\times$ peak \\
 & Gradient clip (norm)       & $1.0$ \\
\midrule
\multirow{4}{*}{Training}
 & Steps                      & $60k$ \\
 & Batch size                 & $128$ \\
 & Checkpoint interval        & $1k$ steps \\
 & Checkpoint evaluated       & step $60k$ \\
\midrule
\multirow{6}{*}{Corruption}
 & Corruption ratio $r$       & $\mathrm{Uniform}[0,1]$ \\
 & Fraction $\to$ \texttt{MASK}             & $0.8$ \\
 & Fraction $\to$ random token              & $0.1$ \\
 & Fraction left unchanged                  & $0.1$ \\
 & Error-sample probability                 & $0.5$ \\
 & Forced fully-masked sample prob.         & $0.1$ \\
\midrule
\multirow{5}{*}{Generation}
 & Diffusion steps $T$        & $24$ (default) \\
 & Mask schedule              & linear \\
 & Temperature $\tau$         & $0.40$ \\
 & Top-$k$ / top-$p$          & unrestricted \\
 & Samples per evaluation     & $200$--$2k$ \\
\bottomrule
\end{tabularx}
\end{table}
\FloatBarrier

\section{Closed-Form Per-Position Distributions}
\label{app:conf}

Because $p$ is known, the per-position distributions a perfectly trained model
reports at the all-masked state are computable. This section derives the entries
in Table~\ref{tab:conf}. Parameters: $n = 9$ values drawn from $\{0, \ldots, 9\}$, $A = 2$
add commands, largest possible answer $AV = 18$.

\paragraph{Step 1: place the two add commands.}
Put the two \texttt{add} commands at positions $i < j$ among the $A + n - 1 = 10$
command positions. There are $\binom{10}{2} = 45$ equally likely pairs. The head
reads $v_i$ and $v_{j-1}$, so $\mathrm{ANS} = v_i + v_{j-1}$. The $9$ pairs
with $j = i + 1$ read one value twice (the head does not advance before the
second add).

\paragraph{Answer hundreds digit $a_2$.}
The largest answer is $18 < 100$, so $a_2 = 0$ in every sample.
\[
  \max_v p(a_2 = v \mid \varnothing) = 1.00, \qquad H(a_2 \mid \varnothing) = 0.
\]

\paragraph{Answer tens digit $a_1$.}
\[
  P(\mathrm{ANS} \geq 10)
  = \underbrace{\tfrac{9}{45} \cdot \tfrac{1}{2}}_{\text{same value twice}}
  + \underbrace{\tfrac{36}{45} \cdot \tfrac{45}{100}}_{\text{two distinct values}}
  = 0.46,
\]
so $a_1 = 1$ with probability $0.46$ and $a_1 = 0$ with probability $0.54$.
\[
  \max_v p(a_1 \mid \varnothing) = 0.54, \qquad H(a_1 \mid \varnothing) = 0.995 \text{ bits.}
\]

\paragraph{Answer units digit $a_0$.}
Computing $P(a_0 = d)$ over all 45 position pairs and 100 value pairs gives
$P(a_0 = d) = 0.12$ for even $d$ and $0.08$ for odd $d$.
\[
  \max_v p(a_0 \mid \varnothing) = 0.12, \qquad H(a_0 \mid \varnothing) = 3.29 \text{ bits.}
\]

\paragraph{Value positions (VALS).}
Each value is drawn i.i.d.\ from $\{0, \ldots, 9\}$, independent of everything else.
\[
  \max_v p(v_i \mid \varnothing) = 0.10, \qquad H(v_i \mid \varnothing) = \log_2 10 = 3.32 \text{ bits.}
\]

\paragraph{Command positions (OPS).}
Each command is \texttt{add} with probability $2/10 = 0.20$ and \texttt{right}
with probability $0.80$.
\begin{itemize}
  \item $M = 1$ (one token id per command): $\max_v p = 0.80$, $H = 0.72$ bits.
  \item $M = 9$ (five synonym ids per command): each id has probability $0.80/5 = 0.16$
        (\texttt{right}) or $0.20/5 = 0.04$ (\texttt{add}).
        $\max_v p = 0.16$, $H = 3.04$ bits.
\end{itemize}

\paragraph{Total correlation of the command block.}
The synonym choice is independent of everything else and cancels in the TC:
\[
  \mathrm{TC}(X_C) = 10\,H(0.2) - \log_2 45 = 7.219 - 5.492 = 1.727 \text{ bits (either } M\text{).}
\]

\paragraph{Total correlation of the two non-degenerate answer digits.}
\[
  \mathrm{TC}(a_1, a_0) = H(a_1) + H(a_0) - H(\mathrm{ANS})
  = 0.995 + 3.293 - 4.099 = 0.189 \text{ bits.}
\]

\paragraph{Rank reflects encoding, not dependence.}
Changing $M$ from $1$ to $9$ (replacing each command id with five synonyms)
leaves the task and its dependency structure unchanged: the command block's
$\mathrm{TC}$ stays at $1.727$ bits. What changes is each command's confidence
($0.80 \to 0.16$), dropping commands below the tens digit in the ranking.
The write order changes; the dependencies do not. A coupled block can therefore
be moved up or down the confidence ranking by a change of representation that
leaves the distribution unchanged.

\section{How Each Diffusion Family Fails on ScanAndAdd}
\label{app:algorithms}

The failure of each family follows from one dependency identified in Section~\ref{sec:instance}.
Details for each family are below; Appendix~\ref{app:conf} derives the confidence
values used.

\subsection{Masked Reveal}
\label{app:masked}

\textbf{Key fact.} Value positions have the lowest confidence in the sequence.
A confidence ranking writes them last, by which time two of them are dependent.

\paragraph{Phase 1: Writing ANS and OPS.}
At the all-masked state, Table~\ref{tab:conf} gives value positions confidence $0.10$, lower
than every other content position. So a confidence ranking writes \texttt{ANS} and
\texttt{OPS} first.

During this phase, every group of two or more undetermined non-marker positions
consists of command positions, answer digits, or both. All of these are dependent
(Section~\ref{sec:instance}), so every parallel step in this phase is inexact.

\paragraph{Phase 2: Writing the value positions.}
Once \texttt{ANS} and \texttt{OPS} are fully written, the nine value positions remain.
Two of them, the \emph{summed pair}, now have elevated confidence:

\begin{itemize}
  \item The recorded answer fixes the sum $s$ of the two values the head reads.
        Each member of the summed pair is then uniform on $k(s) := 10 - |s - 9|$
        values, giving it confidence $1/k(s)$.
  \item $1/k(s) > 0.10$ whenever $s \neq 9$, so both members of the summed pair
        outrank the seven remaining free value positions.
  \item This happens with probability $\tfrac{36}{45} \cdot \tfrac{9}{10} = 0.72$
        (non-adjacent add positions, which give a distinct summed pair, times
        $P(s \neq 9)$).
\end{itemize}

In those cases, the ranking writes the summed pair as the first value-position
parallel step. The summed pair will be one of the following:
\begin{itemize}
  \item \textbf{Dependent:} $\mathrm{TC} = \log_2 k(s)$, up to $3.17$ bits.
  \item \textbf{Entangled:} drawing both members independently from their marginals
        produces the correct sum $s$ with probability only $1/k(s)$, so most
        such steps record an inconsistent answer.
\end{itemize}

\paragraph{Why the sampler cannot avoid this.}
A sampler could write each member of the summed pair in a separate step of size one
and then write the seven free positions together, which would be exact. But
identifying which positions form the summed pair requires knowing where the two
\texttt{add} commands fell, which depends on the sample. Between two and nine
command positions must be revealed before the summed pair is identifiable. Step
sizes fixed in advance cannot always reserve a size-one step at the right moment.

\subsection{Remasking}
\label{app:remask}

\textbf{Key fact.} When a parallel step writes the command block and produces
an invalid add-count, the conditional $p(\cdot \mid x_{-i})$ is undefined at
every position. The remasking rule has nothing valid to read.

\paragraph{When this occurs.}
Ten command positions drawn independently at $P(\texttt{add}) = 0.2$ produce an
add-count of $0$ or $\geq 4$ with probability $0.228$. At either count, branch~(b)
of Theorem~\ref{thm:dich} holds at every position:

\begin{itemize}
  \item \textbf{Count = 0} (no adds): delete any position $i$ from the sequence.
        The remaining nine commands hold zero adds. A valid sequence needs two
        adds total, so even adding an \texttt{add} at position $i$ gives only
        one --- not enough. No valid sequence has this context: $p(x_{-i}) = 0$.

  \item \textbf{Count $\geq 4$}: delete any position $i$. The remaining nine
        hold $\geq 3$ adds. No valid sequence permits more than two adds.
        Again $p(x_{-i}) = 0$.
\end{itemize}

Deleting a value or answer position instead leaves the invalid command block
intact, so $p(x_{-i}) = 0$ at those positions too. The conditional
$p(\cdot \mid x_{-i})$ does not exist anywhere, and whatever the remasking
rule outputs is not determined by $p$.

\paragraph{Other add-counts (1 or 3).}
At count $= 1$ or $3$, branch~(a) applies at some positions: the network reports
zero probability for the token at those positions, identifying them for rewriting.
However, the redraw conditions on the remainder of a command block whose law is
already wrong from the same parallel step.

\subsection{Uniform-State Diffusion}
\label{app:uniform}

\textbf{Key fact.} At every valid ScanAndAdd sequence, every position except the
unread value positions is pinned by the others. Proposition~\ref{prop:frozen} then applies:
writing a pinned position alone leaves the support, and any group that could
change it is entangled.

\paragraph{Which positions are pinned at a valid sequence $x$.}

\begin{center}
\small
\begin{tabular}{ll}
\toprule
\textbf{Position type} & \textbf{Why pinned} \\
\midrule
Command positions & add-count is fixed at two; the other nine determine this one \\
Answer digits ($a_2, a_1, a_0$) & deterministic function of values and commands \\
Markers & constant \\
Summed pair & each member pinned by the recorded answer and the other member \\
\midrule
Free value positions & \textit{not pinned} --- each uniform on ten values \\
\bottomrule
\end{tabular}
\end{center}

\paragraph{What Proposition~\ref{prop:frozen} says.}
At every pinned position, writing that position alone puts the state outside
$\operatorname{supp}(p)$; writing any group that could change it makes that
group entangled, which Theorem~\ref{thm:step} then bounds.

For the command block: moving from one arrangement to another requires changing
at least two command positions at once (the add-count is fixed, so replacing an
\texttt{add} requires adding an \texttt{add} elsewhere). Any such set is
entangled at the frozen tokens. The terminal phase therefore has no exact route
between command arrangements.

\section{The Realized Write Order}
\label{app:trace}

Table~\ref{tab:conf} predicts the write order from the data alone, without any trained network.
Figure~\ref{fig:gen_order} verifies this prediction on trained models ($1k$
generations at two step budgets, $T \in \{10, 25\}$).

\begin{figure}[ht]
\centering
\includegraphics[width=0.48\textwidth]{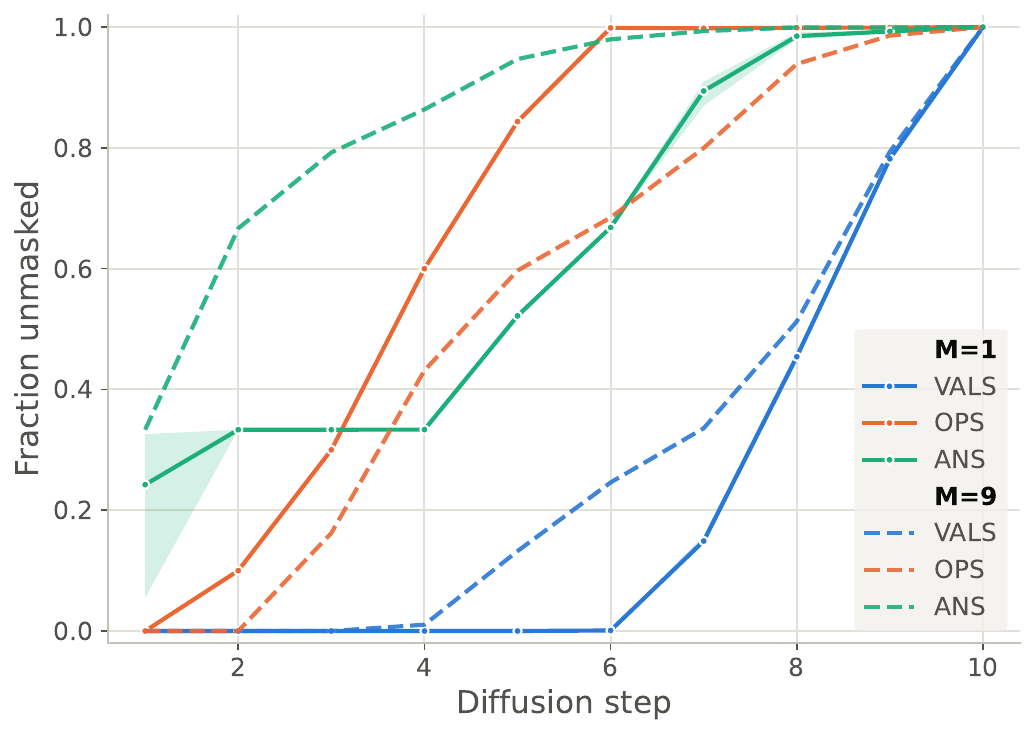}\hfill
\includegraphics[width=0.48\textwidth]{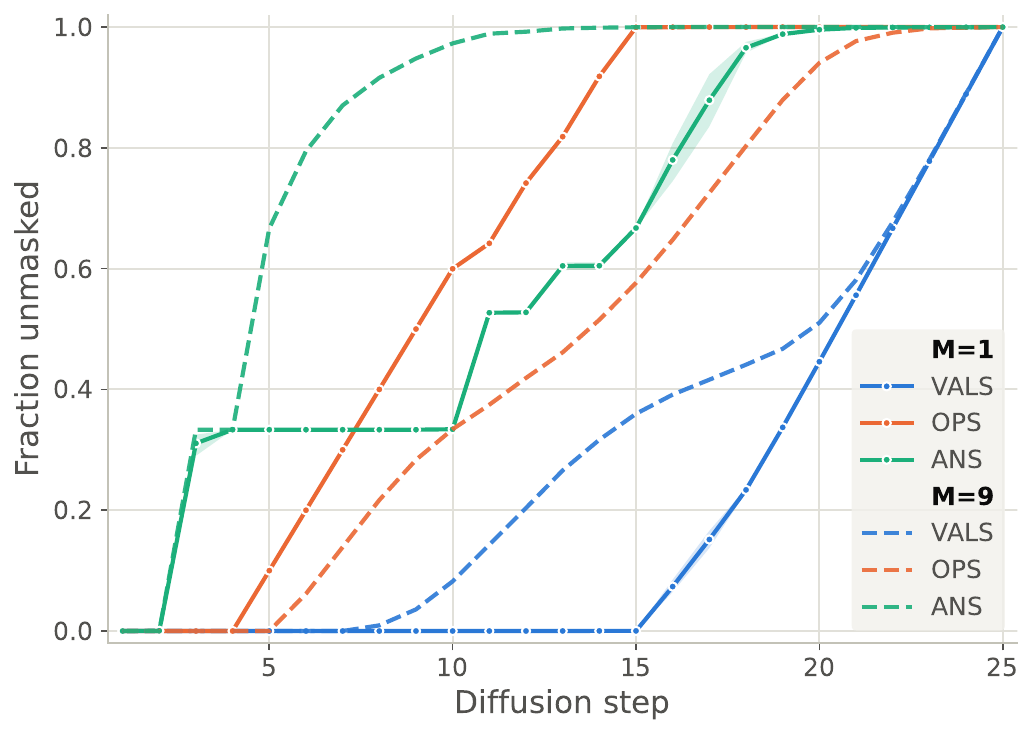}
\caption{Mean fraction of \texttt{VALS}, \texttt{OPS}, and \texttt{ANS} positions
written at each remasking step, over $1k$ generations. \textbf{Left:} $T=10$.
\textbf{Right:} $T=25$. $M=1$ is solid with markers, $M=9$ dashed. Shading is
the $5$--$95$ percentile range across four independently trained $M=1$ models.}
\label{fig:gen_order}
\end{figure}

\paragraph{What the traces show.}
\begin{itemize}
  \item In every condition and at every step budget, the first content position
        written is an \texttt{ANS} digit ($a_2$, confidence $1.00$).
  \item \texttt{VALS} is written last in all conditions.
  \item At $M=1$: \texttt{OPS} finishes before \texttt{VALS} starts, as
        Table~\ref{tab:conf} predicts (command confidence $0.80 > 0.10$).
  \item At $M=9$: \texttt{OPS} and \texttt{VALS} rise together after \texttt{ANS},
        as Table~\ref{tab:conf} predicts (synonym encoding drops command confidence to
        $0.16$, below the tens digit at $0.54$).
  \item The order is consistent across four independently trained $M=1$ models
        (different initialization, data order, and masking noise): the
        $5$--$95$ range is at most $0.071$ for \texttt{ANS} and $0.024$ for
        \texttt{VALS}. The write order is a property of the task and the
        decoding rule, not of any one model.
\end{itemize}

\paragraph{Effect of step budget.}
Halving the step budget (from $T=25$ to $T=10$) does not move value positions
earlier. It forces them into larger groups at the point where they are least
independent --- after the commands and answer have fixed the sum of the summed
pair. The parallelism a tighter budget buys is spent exactly where Theorem~\ref{thm:step}
gives the largest cost.

\FloatBarrier

\section{Errors Add and Cannot Cancel}
\label{app:onlyworse}

Theorem~\ref{thm:step} bounds the error of one step. For the absorbing family (each position
written once), the per-step errors add over the full run and none can offset another.

\paragraph{Condition: prefix-measurable schedule.}
Write $O_t$ for the positions written before step $t$ and $x_{O_t}$ for their values.
Call the schedule \emph{prefix-measurable} if the set $R_t$ written at step $t$ is
a function of $x_{O_t}$ alone.

\begin{itemize}
  \item \textbf{Prefix-measurable:} a schedule fixed in advance; any ranking by
        $\max_v \pi_i(v)$ (confidence based on the current conditional).
  \item \textbf{Not prefix-measurable:} ranking by the probability of a token
        freshly drawn at each candidate position (both the draw and the selected
        token are random).
\end{itemize}

\begin{theorem}[Per-step errors add]\label{thm:add}
Under a prefix-measurable schedule,
\[
  D_{\mathrm{KL}}(p \,\|\, q)
  = \sum_{t} \delta_t, \qquad
  \delta_t
  := \mathbb{E}_{x_{O_t} \sim p}\!\left[
       D_{\mathrm{KL}}\!\left(
         p(x_{R_t} \mid x_{O_t})
         \;\Big\|\;
         \prod_{i \in R_t} \pi_i(\cdot \mid x_{O_t})
       \right)
     \right] \geq 0.
\]
\end{theorem}

\begin{proof}[Proof sketch]
Prefix-measurability gives each sequence $x$ exactly one write path: $R_1$ is
fixed; $R_1$ and $x$ determine $x_{O_2}$ and hence $R_2$; and so on. Along this
path,
\[
  q(x) = \prod_t \prod_{i \in R_t} \pi_i(x_i \mid x_{O_t}), \qquad
  p(x) = \prod_t p(x_{R_t} \mid x_{O_t}).
\]
Take logarithms, subtract, take the expectation under $p$, and condition the $t$th
summand on $x_{O_t}$: the summand becomes the displayed divergence, which is
non-negative. Theorem~\ref{thm:step} further splits each $\delta_t$ into an expected total
correlation and the model's per-position error.
\end{proof}

\paragraph{Three consequences of non-negativity.}
\begin{enumerate}
  \item \textbf{Exactness requires every step to be exact.} The sampler reproduces
        $p$ only if $\delta_t = 0$ for every $t$, i.e., every step is exact at
        $p$-almost every prefix it reaches.
  \item \textbf{One step's error lower-bounds the run.} $D_{\mathrm{KL}}(p \| q)
        \geq \delta_t$ for each $t$: no other step can offset it.
  \item \textbf{A dependent group's cost persists.} A step that writes a dependent
        group on a set of prefixes with positive probability contributes at least
        the total correlation it incurs there.
\end{enumerate}

\paragraph{How errors compound (versus how they add).}
Theorem~\ref{thm:add} evaluates each $\delta_t$ under $p$ --- as though all prior steps were exact. Under 
this accounting, an error at step $k$ leaves the later terms untouched: the errors are additive and separate.

Switching to $q$, which reflects the errors already made, gives the same additivity, but shows something extra: an error at step $k$ carries the sampler to contexts where subsequent conditionals may be worse. Theorem~\ref{thm:dich} is
the extreme case --- a context at which $p$ defines no conditional at all. Neither
accounting admits a negative term, so in neither sense can a later step undo an
earlier one.

\paragraph{What a token-reading scheduler does not cover.}
If $R_t$ is chosen by ranking the probability of a freshly drawn token, the write
path is no longer a function of $x$. Each sequence is reachable along multiple
paths, so $q(x)$ is a sum of products and does not split into per-step terms.
Additivity holds on the space of paths, but data processing only bounds the output
error \emph{above} by the sum of per-step errors. Non-cancellation would require a
lower bound, which this argument does not provide.

Note that cancellation across different departure types is possible: in
Section~\ref{sec:measured}, raising temperature from $0.5$ to $1.5$ reduces the best pooled-token
TV from $0.129$ to $0.0352$, so tempering partially cancels the distortion from
write order. What the results above rule out is one step offsetting another within
a prefix-measurable run.

\section{How Much Parallelism the Exception Permits}
\label{app:schedule}

Corollary~\ref{cor:main} allows exact steps when the group written is conditionally independent.
On ScanAndAdd, one such group exists (the free value positions), and the parallelism
it provides can be computed.

\paragraph{An exact schedule.}
The rule: \emph{write every position at confidence $1.00$ as one group; otherwise
write the single lowest-indexed masked position in field order
\texttt{VALS}, \texttt{OPS}, \texttt{ANS}.}

This schedule produces three types of parallel steps:

\begin{enumerate}
  \item \textbf{First step --- always:} the three markers (constant) and $a_2$
        ($= 0$ in every sample). Four positions, one step.

  \item \textbf{Mid-run --- trailing determined commands:} once the remaining
        unwritten command positions are forced (either both adds are placed, or
        as many positions remain as adds still to place), those trailing commands
        are written together. Enumerating the 45 command arrangements: on average
        $64/9 \approx 7.11$ commands are written one at a time before this point,
        leaving $2.89$ written together. Step count for this phase: $64/9 + 1$.

  \item \textbf{Last step:} once all values and commands are written, $a_1$ and
        $a_0$ are determined. Two positions, one step.
\end{enumerate}

The nine value positions are written one at a time, in \texttt{VALS} field order.

\paragraph{Expected step count.}
\[
  \underbrace{1}_{\text{markers}+a_2}
  + \underbrace{9}_{\text{values, one each}}
  + \underbrace{\tfrac{64}{9}}_{\text{commands, one each}}
  + \underbrace{1}_{\text{trailing commands}}
  + \underbrace{1}_{a_1, a_0}
  = \frac{172}{9} \approx 19.11 \text{ steps},
\]
ranging from $14$ to $21$ across the 45 arrangements. Speedup: $25 / 19.11
\approx 1.31\times$.

\paragraph{Why a confidence ranking cannot use this exception.}
This schedule writes value positions \emph{before} the answer, in ascending order
of confidence. A confidence ranking does the opposite: answer first, values last.
By the time a confidence ranking reaches the value positions, the answer has made
the summed pair dependent, and the exception is unavailable.

A fallback that ranked drawn tokens instead of field order would make the reveal
path random --- a case Appendix~\ref{app:onlyworse} does not cover.

\section{Other Departures from Exact Sampling}
\label{app:interventions}

\paragraph{One criterion covers all cases.}
A sampler with exact conditionals reproduces $p$ if and only if, at every step:
\begin{enumerate}
  \item[(a)] it draws from an unmodified conditional of $p$, and
  \item[(b)] the positions it reveals at that step are conditionally independent
             given what has already been revealed.
\end{enumerate}
Violating (a) is score shaping and enters the per-position error terms of
Theorem~\ref{thm:step}. Violating (b) costs the total correlation of the revealed set and
enters the grouping term of Theorem~\ref{thm:step}, which training does not reduce.

\paragraph{Interventions that always shift.}
These distort for any $p$ with $H(p) > 0$.

\begin{table}[ht]
\centering
\small
\caption{Interventions that shift the distribution for any non-degenerate $p$.}
\label{tab:always}
\begin{tabularx}{\linewidth}{lX}
\toprule
\textbf{Intervention} & \textbf{Why} \\
\midrule
Greedy / argmax decoding           & output is a point mass \\
Beam search                        & targets the joint MAP \\
Best-of-$n$ with a reranker        & reweights by the reranker unless it is constant \\
Classifier-free guidance ($\ne 1$) & renormalized product of two distributions \\
\bottomrule
\end{tabularx}
\end{table}

\paragraph{Interventions that shift under a condition.}
Each is a no-op in an identifiable special case.

\begin{table}[ht]
\centering
\small
\caption{Interventions that shift only when the stated condition holds. $q$
denotes the conditional at the step in question; $R$ is the set revealed.}
\label{tab:conditional}
\begin{tabularx}{\linewidth}{lX}
\toprule
\textbf{Intervention} & \textbf{Shifts if and only if} \\
\midrule
Temperature $\tau \ne 1$             & some conditional is non-uniform on its support \\
Top-$k$                              & $k < |\operatorname{supp}(q)|$ at some step \\
Top-$p$ / nucleus                    & the nucleus omits some support element \\
Min-$p$, $\epsilon$-, typical        & some support token falls below the threshold \\
Parallel reveal                      & $\mathrm{TC}(X_R \mid x_{O_t}) > 0$ for some $R$ \\
Fewer steps than positions ($T < L$) & some step co-reveals a dependent set \\
Remasking                            & same condition; the remasking kernel is not the issue \\
Constrained / grammar decoding       & the constraint set $\mathcal{C}$ has $p(\mathcal{C}) < 1$ \\
\bottomrule
\end{tabularx}
\end{table}

\paragraph{Temperature on ScanAndAdd.}
Temperature leaves a conditional unchanged if and only if that conditional is
uniform on its support.

\begin{itemize}
  \item Under a dependency-respecting order (values written before answer): value
        conditionals are uniform at $0.10$ per digit, so temperature only affects the command and answer
        distributions.
  \item Under an answer-first order: value conditionals are no longer uniform (the
        answer constrains the summed pair). Temperature distorts them as well.
\end{itemize}

The same temperature parameter is inert or distorting depending on the write order.

\paragraph{Writing the answer before the values.}
Done one position at a time with exact conditionals, this is a valid chain-rule
factorization of $p$ and introduces no approximation by itself. Its cost comes
indirectly: it makes the summed value pair dependent, so any later parallel step
writing both of them leaves the support. The cost is charged to the two terms of
Theorem~\ref{thm:step}.

\paragraph{Entropy split (ScanAndAdd only).}
On ScanAndAdd, $p$ is uniform on its support, so any sampler that stays inside
the support without reproducing $p$ has strictly lower entropy. This splits
departures into two kinds:

\begin{table}[ht]
\centering
\small
\caption{How each departure moves the generated distribution on ScanAndAdd.
The entropy split uses uniformity of $p$ and does not hold for general $p$.
KL costs hold generally.}
\label{tab:departures}
\begin{tabularx}{\linewidth}{p{3.6cm}llX}
\toprule
\textbf{Departure} & \textbf{Entropy vs.\ $H(p)$} & \textbf{Support} & \textbf{Basis} \\
\midrule
Temperature $\tau < 1$      & lower                            & inside                  & entropy is Schur-concave \\
Top-$k$ / nucleus           & lower                            & inside                  & entropy is Schur-concave \\
Greedy / MAP                & $\to 0$                          & inside                  & limit of the above \\
Aggregate before its inputs & unchanged                        & leaves once grouped     & Section~\ref{sec:instance} \\
Writing a dependent group   & \textbf{higher, by }$\mathrm{TC}$ & \textbf{leaves, if entangled} & Theorem~\ref{thm:step} \\
\bottomrule
\end{tabularx}
\end{table}

\paragraph{Interventions that never shift.}
With exact conditionals, the rules below reproduce $p$. All assume one position
per step.

\begin{table}[ht]
\centering
\small
\caption{Interventions that preserve the distribution with exact conditionals
(one position per step). The first row's assumption matters: reveal order costs
nothing, but the choice of \emph{which} positions to reveal together does, and
no scheduler reading per-position distributions can certify that choice
(Theorem~\ref{thm:cert}).}
\label{tab:never}
\begin{tabularx}{\linewidth}{lX}
\toprule
\textbf{Intervention} & \textbf{Why} \\
\midrule
Any reveal order, one position per step  & chain-rule factorization of $p$; one-per-step is the binding constraint \\
Speculative decoding, exact verification & distribution-preserving by construction \\
$\tau = 1$ with $k \geq |\operatorname{supp}(q)|$ & no-ops \\
\bottomrule
\end{tabularx}
\end{table}

\ifneurips\else
\applefootnote{ \textcolor{textgray}{\sffamily Apple and the Apple logo are trademarks of Apple Inc., registered in the U.S. and other countries and regions.}}
\fi

\end{document}